\documentclass{article}

\usepackage{microtype}
\usepackage{graphicx}
\usepackage{subfigure}
\usepackage{amsmath}
\usepackage{amssymb}
\usepackage{booktabs} % for professional tables

\usepackage{comment}
\usepackage{amsthm} 
\usepackage{tikz}
\usepackage{url}
\usepackage{grffile}

\usepackage{hyperref}

\newcommand{\Exp}{\mathbf{E}}

\newcommand{\R}{{\mathbb R}}

\newcommand{\cH}{{\cal H}}

\newcommand{\bx}{{\bf x}}
\newcommand{\by}{{\bf y}}
\newcommand{\bX}{{\bf X}}
\newcommand{\bY}{{\bf Y}}
\newcommand{\bZ}{{\bf Z}}

\newcommand{\cF}{{\cal F}}
\newcommand{\cN}{{\cal N}}

\newcommand{\ba}{{\bf a}}
\newcommand{\bc}{{\bf c}}

\newcommand{\bz}{{\bf z}}

\newcommand{\sx}{\Sigma_{\bX\bX}}

\newcommand{\sz}{\Sigma_{\bZ\bZ}}

\newcommand{\sxy}{\Sigma_{\bX Y}}
\newcommand{\sxe}{\Sigma_{\bX E}}

\newtheorem{Theorem}{Theorem}
\newtheorem{Lemma}{Lemma}
\newtheorem{Definition}{Definition}

\usetikzlibrary{arrows,shapes,plotmarks,positioning}
\usetikzlibrary{decorations.markings}

\tikzset{>=stealth'} 
\tikzstyle{graphnode} = 
   [circle,draw=black,minimum size=22pt,text centered,text
     width=22pt,inner sep=0pt] 
\tikzstyle{var}   =[graphnode,fill=white]
\tikzstyle{vardashed}   =[graphnode,draw=gray,fill=white]
\tikzstyle{obs}   =[graphnode,fill=black,text=white]
\tikzstyle{obsgrey}   =[graphnode,draw=white,fill=lightgray,text=black]
\tikzstyle{par}    =[graphnode,draw=white,fill=red,text=black] 
 \tikzstyle{crucial} =[graphnode,draw=white,fill=yellow,text=black] 
\tikzstyle{fac}   =[rectangle,draw=black,fill=black!25,minimum size=5pt]
\tikzstyle{facprior} =[rectangle,draw=black,fill=black,text=white,minimum size=5pt]
\tikzstyle{edge}  =[draw=white,double=black,very thick,-]
\tikzstyle{blueedge}  =[draw=white,double=blue,very thick,-]
\tikzstyle{rededge}  =[draw=white,double=red,very thick,-]
\tikzstyle{prior} =[rectangle, draw=black, fill=black, minimum size=
5pt, inner sep=0pt]
\tikzstyle{dirprior} = [circle, draw=black, fill=black, minimum
size=5pt, inner sep=0pt]

\tikzstyle{dot_node}=[draw=black,fill=black,shape=circle]

\usepackage[accepted]{arxiv_from_icml}
\usepackage{url}

\icmltitlerunning{Causal Regularization}

\begin{document}

\twocolumn[
\icmltitle{Causal Regularization}



%\icmlsetsymbol{}{}

\begin{icmlauthorlist}
\icmlauthor{Dominik Janzing}{1}
\end{icmlauthorlist}

\icmlaffiliation{1}{Amazon Development Center, T\"ubingen, Germany}

\icmlcorrespondingauthor{}{janzind@amazon.com}
%\icmlcorrespondingauthor{Eee Pppp}{ep@eden.co.uk}

\icmlkeywords{Machine Learning, ICML}

\vskip 0.3in
]



\printAffiliationsAndNotice{\icmlEqualContribution} % otherwise use the standard text.

\pagestyle{plain}

\begin{abstract}
I argue that regularizing terms in standard regression methods not only help against overfitting {\it finite} data, but sometimes
also yield better  {\it causal} models in the infinite sample regime. 
I first  consider a multi-dimensional variable  linearly influencing a target variable with some multi-dimensional unobserved common cause, where the confounding effect can be decreased by keeping the penalizing term in Ridge and Lasso regression even in the population limit. 
Choosing the size of the penalizing term, is however challenging, because cross validation is pointless.
Here it is done 
 by first estimating the strength of confounding via a method proposed earlier, which yielded some reasonable results for simulated and real data.

Further, I prove a `causal generalization bound'   which states   (subject to a particular model of confounding)  
that the error made by interpreting any non-linear regression as {\it causal} 
model can be bounded from above whenever functions are taken from a not too rich class. 
In other words, the bound guarantees `generalization' from {\it observational} to {\it interventional} distributions,
 which is usually not subject of statistical learning theory (and is only possible due to the underlying symmetries of the
 confounder model).  
\end{abstract}

\section{Introduction}

Predicting a scalar target variable $Y$ from a $d$-dimensional predictor $\bX:=(X_1,\dots,X_d)$ via appropriate regression models is among the classical problems 
of machine learning, see e.g. \cite{Schoelkopf2002}. In the standard supervised learning scenario, some finite number of observations, independently drawn from 
an unknown but fixed joint distribution $P_{Y,\bX}$, are used for inferring $Y$-values corresponding to unlabelled $\bX$-values. 
To solve this task, regularization is known to be crucial for obtaining regression models 
that generalize well from training to test data \cite{Vapnik}.  
Deciding whether such a regression model admits a {\it causal} interpretation is, however, challenging. Even if
causal influence from $Y$ to $\bX$ can be excluded (e.g. by time order), the statistical relation between $\bX$ and $Y$ cannot necessarily be attributed to the influence of $\bX$ on $Y$. Instead, it could be due to possible common causes, also called `confounders'.  For the case where common causes are known and observed, there
is a huge number of techniques to infer the causal influence\footnote{often for $d=1$ and with a binary treatment variable $X$}, 
e.g.,  \citep{Rubin}, addressing different challenges, for instance, high dimensional confounders \citep{Chernuzkov18} or the case where some 
variables other than the 
common causes are observed \citep{Pearl:00}, just to mention a few of them. If common causes are not known, the task of inferring the influence of $\bX$ on $Y$ gets 
incredibly hard. Given observations from any further variables other than $\bX$ and $Y$, conditional independences may help to detect or disprove the existence of common causes \citep{Pearl:00}, and so-called instrumental variables may admit the identification of causal influence \citep{Imbens94}. 

Here we consider the case where only observations from $\bX$ and $Y$ are given. In this case, naively interpreting the regression model
as causal model is a natural baseline. Within our simplified scenario, we show that strong regularization increases the chances that the regression model contains some causal truth. I am aware of the risk that this result could be mistaken as a justification to ignore the hardness of the problem and blindly infer causal models by strong regularization. My goal is, instead, to inspire a discussion on to what extent causal modelling should regularize even in the infinite sample limit
due to some analogies between {\it generalizing across samples from the same distribution} and {\it `generalizing' from observational to interventional distributions},
which appear in a particular model of confounding, while they need not apply to other confounding scenarios. The idea that regularization can also help for better generalization across different environments rather than only across different subsamples from the same distribution can already be found in the literature
\cite{HeinzeDeml2017b}, but here I describe a model of confounding for which the analogy between confounding and overfitting is so tight that exactly the
same techniques help against both. 

\paragraph{Scenario 1: inferring a linear statistical model} To explain the idea, we consider the simple case where the 
statistical relation between $\bX$ and $Y$ is given by the linear model
\begin{equation}\label{eq:simplelinear}
Y = \bX \ba + E,
\end{equation}
where $\ba$ is a column vector in $\R^d$ and $E$ is an uncorrelated unobserved noise variable, i.e.,   $\sxe =0$. 
We are given observations from $\bX$ and $Y$. Let $\hat{Y}$ denote the column vector of centred renormalized observations $y^i$ of $Y$, i.e., with entries $(y^i-\frac{1}{n}\sum_{i=1}^n y^i)/\sqrt{n-1}$, and similarly, $\hat{E}$ denotes the centred renormalized values of $E$. Likewise, let
$\hat{\bX}$ denote the $n\times d$ matrix whose $j$th column contains the centred renormalized observations from $X_j$. 
Let, further,  $\hat{\bX}^{-1}$ denote its (Moore-Penrose) pseudoinverse.  
To avoid overfitting, the least squares estimator\footnote{Here we have, for simplicity, assumed $n> d$.} 
\begin{equation}\label{eq:lsq}
\hat{\ba}:= {\rm argmin}_\ba'  \|\hat{Y} -  \hat{\bX} \ba' \|^2 = \hat{\bX}^{-1} \hat{Y} = \ba + \hat{\bX}^{-1} \hat{E},
\end{equation}  
is replaced with the 
Ridge and Lasso estimators
\begin{eqnarray}
\hat{\ba}^{\rm ridge}_\lambda &:=& {\rm arg min}_{\ba'} \{ \lambda \|\ba'\|^2_2 +  \|\hat{Y}  - \hat{\bX} \ba'\|^2 \} \label{eq:ridge}\\
\hat{\ba}^{\rm lasso}_\lambda &:=& {\rm arg min}_{\ba'} \{ \lambda \|\ba'\|_1 + \|\hat{Y} - \hat{\bX} \ba'\|^2 \} \label{eq:lasso}, 
\end{eqnarray}
where $\lambda$ is a regularization parameter \citep{Hastie2008}. 

%For our purposes, it is helpful to rewrite \eqref{eq:ridge} and \eqref{eq:lasso} in terms of sample covariance matrices. Using the projection $\bx \bx^{-1}$ 
%we decompose the error term into
%\begin{eqnarray*}
%\|y-\bx \ba'\|^2 &=& \|\bx \bx^{-1} y -\bx\ba'\|^2 + \|(I-\bx\bx^{-1})y\|^2. 
%\end{eqnarray*}
%&=& \langle (\bx^{-1}y - \ba'),\bx^T \bx (\bx^{-1}y - \ba')\rangle \\
%&=& \langle (\hat{\ba} -\ba'), \widehat{\sx} (\hat{\ba} -\ba')\rangle,    
%\end{eqnarray*}
%The first term can be rewritten as
%\[
%\langle (\hat{\ba} -\ba'), \widehat{\sx} (\hat{\ba} -\ba')\rangle,
%\]
%while the second part is irrelevant for the minimization since it does not depend on $\ba'$.
%We thus have 
%\[
%\hat{\ba}_\lambda = \
%{\rm arg min}_{\ba'} \left\{  \begin{array}{c} 
%\lambda \|\ba'\|^2_2 +   \langle (\hat{\ba} -\ba'), \widehat{\sx} (\hat{\ba} -\ba')\rangle  \\
 %\lambda \|\ba'\|_1 +    \langle (\hat{\ba} -\ba'), \widehat{\sx} (\hat{\ba} -\ba')\rangle       \end{array}\right\} 
%\]
%In other words, Ridge and Lasso trade off the $\ell_1$ or $\ell_2$ norm of the regression vector with its square distance from the unregularized vector (w.r.t.~a
%sample-covariance matrix induced inner product). 

So far we have only described the standard scenario of inferring properties of the conditional $P_{Y|X}$ from finite observations $\hat{\bX},\hat{Y}$ without any {\it causal} semantics.

\paragraph{Scenario 2: inferring a linear causal model}
We now modify the scenario in three respects. First, we assume that
 $E$ and $\bX$  in \eqref{eq:simplelinear} correlate due to some unobserved common cause.  Second, we  interpret  \eqref{eq:simplelinear} 
in a {\it causal way} in the sense that setting $\bX$ to $\bx$ lets $Y$ being distributed according to $ \bx \ba +E$.
Using Pearl's do-notation \citep{Pearl:00}, this can be phrased as
\begin{equation}\label{eq:intob}
Y|_{do(\bX=\bx)} = \bx \ba + E  \quad \neq Y|_{\bX=\bx}. 
\end{equation}
Third, we assume the infinite sample limit where $P_{\bX,Y}$ is known. We still want to infer $\ba$ because we are interested in causal statements 
but regressing $Y$ on $\bX$ yields $\hat{\ba}$ instead which describes the  {\it observational} conditional on the right hand side of \eqref{eq:intob}.

Conceptually, Scenario 1 and 2 deal with two entirely different problems: inferring $P_{Y|{X=x}}$  from finite samples $(\hat{\bX},\hat{Y})$ versus inferring the interventional conditional $P_{Y|do(\bX=\bx)}$ from the observational distribution $P_{Y,\bX}$. However, in our case, both problems amount to inferring the vector $\ba$. Further, for both scenarios, the error term $\hat{\bX}^{-1} \hat{E}$ is responsible for the failure of ordinary least squares regression.  Only the reason why this term is non-zero differs: in the first scenario it is a finite sample effect, while it results from confounding in the second one. The idea of the present paper is simply that standard regularization techniques do not care about the {\it origin} of this error term. Therefore, they can temper the impact of confounding in the same way as they help avoiding to overfit finite data.  Such a strong statement, for course, relies heavily on our highly idealized generating model for the confounding term. We therefore ask the reader not to quote it without also mentioning the strong assumptions. 
%In general, inferring $P_{Y|do(\bX=\bx)}$ from $P_{\bX,Y}$ is incredibly hard. After all, popular scenarios of causal inference infer $P_{Y|do(\bX=\bx)}$
% when not only observations from $\bX,Y$ but also from the common causes are given \cite{Pearl:00,Rubin}. 

 The paper is structured as follows. Section~\ref{sec:overfitconfound} elaborates on the analogy between overfitting and confounding by only slightly generalizing 
 observations of \citet{JanSch18}. Section~\ref{sec:popRidgeLasso} describes population versions of Ridge and Lasso regression and provides a Bayesian justification. 
 %Section~\ref{sec:generalization} describes that generalizing between different samples from the same distribution for scenario 1 is analog to generalizing across
 %samples from different background conditions for scenario 2. 
 Section~\ref{sec:est} proposes a way to determine the regularization parameter in scenario 2 by estimating the
 strength of confounding via a method proposed by \citet{JanSch18}. 
 Section~\ref{sec:exp} describes some empirical results. 
 Section~\ref{sec:learningtheory} describes a modified statistical learning theory that states that 
 regression models from not too rich function classes `generalize' from statistical to causal statements. 
 %Section~\ref{sec:outlook} speculates about which part of our ideas still hold when our particularly simple generating model for confounding is violated. 

\section{Analogy between overfitting and confounding \label{sec:overfitconfound}}

The reason why our scenario 2 only considers the {\it infinite} sample limit of confounding is that mixing finite sample and confounding significantly complicates the 
theoretical results. 
For a concise description of the population case, we consider the Hilbert space $\cH$ of centred random variables (on some probability space without further specification) with finite variance. The inner product is given by the covariance, e.g.,
\begin{equation}\label{eq:hilbertinner}
\langle X_i, X_j\rangle := {\rm cov}(X_i,X_j). 
\end{equation}
Accordingly, we can interpret $\bX$ as an operator\footnote{Readers not familiar with operator theory may read all our operators as matrices with 
huge $n$ without loosing any essential insights -- except for the cost of having to interpret all equalities as {\it approximate} equalities. To facilitate this way of reading, we will use $(\cdot)^T$ also for the adjoint of operators in $\cH$ although $(\cdot)^*$ or $(\cdot)^\dagger$ is common.}   $\R^d\to \cH$ via $(b_1,\dots,b_d) \mapsto \sum_j b_j X_j$.   
Then the population version of \eqref{eq:lsq} reads
\begin{equation}\label{eq:atilde}
\tilde{\ba} =  {\rm argmin}_\ba'  \{\| Y -  \bX \ba' \|^2 \}= \bX^{-1} Y = \ba + \bX^{-1} E,
\end{equation}
where the square length is induced by the inner product \eqref{eq:hilbertinner}, i.e., it is simply the variance. Extending the previous notation, $\bX^{-1}$ now denotes the pseudoinverse 
%\footnote{For any $f$ in the image of $\bX$, $X^{-1} f$ is the solution of $X g =f$ with minimal norm \cite{Christensen95}.} 
of
the operator $\bX$ \citep{Beutler65}.
To see that 
$\bX^{-1} E$ is only non-zero when $\bX$ and $E$ are correlated it is helpful to rewrite it as
\begin{equation}\label{eq:sxe}
\bX^{-1} E= \sx^{-1} \sxe,
\end{equation}
where we have assumed $\sx$ to be invertible (see appendix for a proof).
The claim that standard regularization like Ridge and Lasso work for tempering the impact of the term $\bX^{-1} E$ in the same way as they work for
$\hat{\bX}^{-1} \hat{E}$ is inspired by the observation of 
\citet{JanSch18} that these terms follow the same distribution subject to the idealized model assumptions described there. We obtain the same analogy for a slightly more general
model which we describe now.

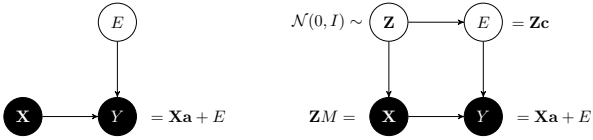
\begin{figure}
\begin{center}
\resizebox{8cm}{!}{
  \begin{tikzpicture}
    \node[obs] at (0,0) (X) {$\bX$} ;
    \node[obs] at (2,0) (Y) {$Y$} edge[<-] (X);
    \node[var] at (2,2) (E) {$E$} edge[->] (Y);  
    \node[anchor=center] at (3.5,0) {$= \bX \ba + E$};
  \end{tikzpicture}
  \hspace{1cm}
    \begin{tikzpicture}
    \node[obs] at (0,0) (X) {$\bX$} ;
    \node[anchor=center] at (-1.2,0) {$\bZ M =$};
    \node[obs] at (2,0) (Y) {$Y$} edge[<-] (X);
    \node[anchor=center] at (3.5,0) {$= \bX \ba + E$};
    \node[var] at (2,2) (E) {$E$} edge[->] (Y);  
    \node[anchor=center] at (3,2) {$= \bZ \bc$};
    \node[var] at (0,2) (Z) {$\bZ$} edge[->] (X) edge[->] (E); 
    \node[anchor=center] at (-1.3,2) {$\cN(0,I)  \sim$};
  \end{tikzpicture}
}
\caption{\label{fig:dags} Left: In scenario 1, the empirical correlations between $X$ and $E$ are only finite sample effects.  Right:
In scenario 2, $\bX$ and $E$ are correlated due to their common cause $\bZ$. We sample the structural parameters $M$ and $\bc$ from   
distributions in a way that entails a simple analogy between scenario 1 and 2.} 
\end{center}
\end{figure}

\paragraph{Generating model for scenario 1}
The following procedure generates samples according to the DAG in Figure~\ref{fig:dags}, left:

\begin{tabular}{ll}
1. & Draw $n$ observations from $(X_1,\dots,X_d)$ \\ 
    & independently from $P_X$ \\
2. & Draw samples of $E$ independently from $P_E$\\
3. & Draw the vector $\ba$ of structure coefficients\\
    &  from some distribution $P_\ba$ \\
4. & Set $\hat{Y} := \hat{\bX}\ba + \hat{E}$.\\ 
\end{tabular}

\paragraph{Generating model for scenario 2} 

To generate random variables according to the DAG in Figure~\ref{fig:dags}, right, we assume that both variables $\bX$ and $E$ are generated
from the same set of independent sources by applying a random mixing matrix or a random mixing vector, respectively:

Given an $\ell$-dimensional random vector $\bZ$ of sources with distribution $\cN(0,I)$. 

\begin{tabular}{ll}
1 . & Choose an $\ell \times d$ mixing matrix $M$ \\
 & and set
%\begin{equation}\label{eq:MZ}
$\bX := \bZ M$.
%\end{equation}
\\
2. & Draw $\bc\in \R^\ell$ from some distribution $P_\bc$ \\
 & and set
$
E:=  \bZ\bc.
$ \\
3. & Draw the vector $\ba$ of structure coefficients\\
&  from some distribution $P_\ba$\\
4. &  Set 
%\begin{equation}\label{eq:Y}
$Y:= \bX \ba + E $.
%\end{equation}
\end{tabular}

We then obtain:
%\begin{Lemma}[equivalence of error vectors]\label{lem:errorequiv}
%Let $M$ in scenario 2 coincide with $\hat{\bX}$ in scenario 1 (which implies that the sample size $n$ translates into the  number $\ell$ of sources).
%Generate  $\bc$  in scenario 2 by drawing each entry independently  from $P_E$ from scenario 1 and dividing it by $\sqrt{\ell-1}$, then  
%$\bX^{-1}E$ in scenario 2 follows the same distribution as $\hat{\bX}^{-1}\hat{E}$ in scenario 1.
%\end{Lemma}   
%\begin{proof}
%We have $\bX^{-1} E = (\bZ M)^{-1} \bZ \bc =M^{-1} \bZ^{-1} \bZ \bc = M^{-1} \bc$, where we have used that the operator $\bZ$ has full rank.
%$M^{-1}\bc$ follows the same distribution as 
%$\hat{\bX}^{-1} \hat{E}$: First, $M$ follows the same distribution as $X$, and second, $\bc$ follows the same distribution as $\hat{E}$ after centering, 
%but centering does not change $ \hat{\bX}^{-1} \bc$ because it only affects
%the component in the kernel of $\hat{\bX}^{-1}$. 
%\end{proof}
%Moreover, we have:
\begin{Theorem}[population and empirical covariances]
\label{thm:equivalence}
Let the number $\ell$ of sources in scenario 2 be equal to the number $n$ of samples in scenario 1 and 
 $P_M$ coincide  with the distribution of sample matrices $\hat{\bX}$ induced by $P_\bX$. Let, moreover, $P_\bc$ in scenario 2 coincide with the distribution
 of $\hat{E}$ induced by $P_E$ in scenario 1, and $P_\ba$ be the same in both scenarios. 
 Then the joint distribution of $\ba,\sx,\sxy,\sxe$ in scenario 2 coincides with the joint distribution of
$\ba,\widehat{\sx},\widehat{\sxy},\widehat{\sxe}$ in scenario 1.
\end{Theorem}
\begin{proof}
We have $\widehat{\sx} = \hat{\bX}^T \hat{\bX}$ and $\sx = \bX^T \bX = M^T \bZ^T \bZ M = M^T M$,
where we have used that $\bZ$ has full rank due to the uncorrelatedness of its components. 
Likewise,
$\widehat{\sxe} = \hat{\bX}^T \hat{E}$ and $\sxe = (\bZ M)^T \bZ \bc = M^T \bc$. Further,
$\widehat{\sxy} = \hat{\bX}^T \bX \ba + \hat{\sxe}$ and $
\sxy = \bX^T \bX\ba + \sxe$.
Then the statement follows from the correspondences
$M \equiv \hat{\bX}$, $\bc \equiv \hat{E}$, $\ba \equiv \ba$.   
\end{proof}
Theorem~\ref{thm:equivalence} provides a canonical way to transfer methods for inferring the vector $\ba$ from empirical covariance matrices in scenario 1
to similar methods for inferring $\ba$ in scenario 2 from population covariance matrices.
Motivated by this insight we will now develop our `causal' Ridge and Lasso for the population case. To emphasize that this method
uses weaker assumptions than Theorem~\ref{thm:equivalence}, we will not strictly build on it and use a more abstract condition
that is only motivated by the concrete model above. 

\section{Bayesian justification for Ridge and Lasso in scenario 2 \label{sec:popRidgeLasso}}

We now define population versions of Ridge and Lasso that temper confounding in the same way as the usual versions temper overfitting.
\begin{eqnarray}
\tilde{\ba}^{\rm ridge}_\lambda &:=& {\rm arg min}_{\ba'} \{ \lambda \|\ba'\|^2_2 +  \|Y - \bX \ba'\|^2 \} \label{eq:ridgep}\\
\tilde{\ba}^{\rm lasso}_\lambda &:=& {\rm arg min}_{\ba'} \{ \lambda \|\ba'\|_1 + \|Y - \bX \ba'\|^2\}. \label{eq:lassop}
\end{eqnarray}
We briefly sketch standard Bayesian arguments for the finite sample versions \cite{Hoerl2000}. Let the prior distributions for $\ba$ be given by
\begin{eqnarray}
p_{\rm ridge}(\ba) &\sim & \exp \left( -\frac{1}{2\tau^2} \|\ba\|^2 \right) \label{eq:aridge}\\
p_{\rm lasso}(\ba) & \sim &  \exp \left(-\frac{1}{2\tau^2} \|\ba\|_1\right) \label{eq:alasso}.
\end{eqnarray}
If we assume that the noise variable $E$ is Gaussian with standard deviation $\sigma_E$ we obtain
\[ 
p(\by|\bx,\ba) \sim \exp \left( -\frac{1}{2\sigma_E^2}\|\by - \bx \ba\|^2 \right),
\]
which yields the posteriors
\begin{align}\label{eq:postridge}
\log p_{\rm ridge}(\ba | \hat{\bX},\hat{Y}) &\stackrel{+}{=}    -\frac{1}{2\tau^2} \|\ba\|^2 -\frac{1}{2\sigma_E^2}\|\hat{Y} - \hat{\bX} \ba\|^2\\
\log p_{\rm lasso}(\ba | \hat{\bX},\hat{\bY}) & \stackrel{+}{=}   -\frac{1}{2\tau^2} \|\ba\|_1 -\frac{1}{2\sigma_E^2}\|\hat{Y} - \hat{\bX} \ba\|^2, \label{eq:postlasso}
\end{align}
which are maximized for $\hat{\ba}_\lambda$ in \eqref{eq:ridge} and \eqref{eq:lasso}, respectively, after setting $\lambda = \sigma^2_E/\tau^2$ (here $\stackrel{+}{=}$ denotes equality up to an additive $\ba$-independent term).   

To derive the posterior for $\ba$ for scenario 2, we recall that now the entire distribution $P_{\bX,Y}$ is given. We also know that $\bX$ ad $Y$ are related by
$Y =\bX \ba +E$, but $\ba$ and $E$ are unknown. For $\ba$ we will adopt the priors  \eqref{eq:aridge} and \eqref{eq:alasso}, but to 
define a reasonable prior for  $E$ is less obvious. Note that we are not talking about a prior for the {\it values} attained by $E$. Instead, $E$ is an unknown vector
in the infinite dimensional Hilbert space $\cH$.  Fortunately, we do not need to specify
a prior for $E$, it is sufficient to specify a prior for the projection $E_\bX$ onto the image of $\bX$. We assume:
\begin{equation}\label{eq:EX}
E_\bX \sim \cN(0,\sigma_{E_\bX}^2 {\bf I}) 
\end{equation}
for some `confounding parameter' $\sigma_{E_\bX}^2$. 
This implies the following distribution for the projection $Y_\bX$ of $Y$ onto the image of $\bX$:
\[
p(Y_\bX|\bX,\ba) \sim  \exp \left( -\frac{1}{\sigma_{E_\bX}^2} \|Y_\bX - \bX \ba \|^2 \right).
\]
This way, we obtain the following posteriors for $\ba$: 
\begin{align}\label{eq:postridgep}
\log p_{\rm ridge}(\ba | \bX, Y_\bX)  & \stackrel{+}{=}-\frac{1}{2\tau^2} \|\ba\|^2 -\frac{1}{2\sigma_{E_\bX}^2}\|Y_\bX - \bX \ba\|^2 \\
\log p_{\rm lasso}(\ba | \bX, Y_\bX)  &\stackrel{+}{=}  -\frac{1}{2\tau^2} \|\ba\|_1 -\frac{1}{2\sigma_{E_\bX}^2}\|Y_\bX - \bX \ba\|^2 . \label{eq:postlassop}
\end{align}
After replacing $Y_\bX$ with $Y$ (which is irrelevant for the maximization) we observe that the posterior probabilities are maximized by \eqref{eq:ridgep} and
\eqref{eq:lassop} with $\lambda:=\sigma_{E_\bX}^2/\tau^2$. 

We phrase our findings as a theorem:
\begin{Theorem}[justification of population Ridge and Lasso] \label{thm:justPopulation} 
Given a $d$-dimensional random variable $\bX$ and a scalar random variable $Y$ for which $P_{\bX,Y}$ is known.
Let they be linked by
\[
Y = \bX \ba + E,
\]
where $\ba\in \R^d$ is unknown and $E$ is a random variable whose distribution is unknown.  
Assume \eqref{eq:aridge} and \eqref{eq:alasso} as priors for $\ba$, respectively. Assume that the projection $E_\bX$ of $E$ on the image of $\bX$  follows
the prior distribution 
\[
E_\bX \sim \exp \left(-\frac{1}{2 \sigma_{E_\bX}^2} \|E_\bX\|^2\right).
\]
 Then the posterior probability $p(\ba|P_{\bX,Y})$ is maximized by
the population Ridge and Lasso estimators \eqref{eq:ridgep} and \eqref{eq:lassop}, respectively, for $\lambda:=\sigma_{E_\bX}^2/\tau^2$.
\end{Theorem}
Here we decided to write
$p(\ba |P_{\bX,Y})$ instead of $p(\ba|\bX,Y)$ to avoid that it could be misunderstood as the conditional given {\it observations from} $\bX,Y$ instead of the entire
statistics\footnote{Actually, only the covariance matrices $\sx,\sxy$ matter, as shown in the appendix.}. In the derivations above it was convenient to keep them as similar to the finite sample case as possible by simply removing the  symbol $\hat{}$. 

%It would be nice to also have a finite sample version of scenario 2, but we didn't find a way to combine them nicely.
The results raise the questions how to select $\lambda$ for our population Ridge and Lasso.
First note that  information criteria like AIC and BIC  \citep{Bandyopadhyay2011} cannot be applied: since they require the sample size in scenario 1, they would require the number of sources in scenario 2, which we assume to be unknown (assuming it to be known seems to go too far away from real-world scenarios). 
To focus on another standard approach of choosing $\lambda$, note that transferring {\it cross-validation} \citep{Hastie2008} from scenario 1 to scenario 2 requires data from different distributions\footnote{as in `invariant prediction' \citep{Peters15}} 
(recall that drawing $\hat{\bX},\hat{E}$ corresponds to
drawing $M$ and $\bc$ 
in Section~\ref{sec:overfitconfound}), which we do not assume to be available here.
%\footnote{It is meanwhile more and more believed that {\it causal} model generalize better across datasets from different background conditions, see e.g. \citet{Peters15,causal_image_classification}.} 
Therefore we need to estimate the strength of confounding to choose
the regularization constant.

\section{Choosing the regularization constant by estimating confounding \label{sec:est}}
%we propose to run our causal Ridge and Lasso with large values of $\lambda$
%whenever confounding is strong. 
The only approaches that directly estimate the strength of confounding\footnote{\citet{HoyerLatent08} construct confounders for linear non-Gaussian
models and \citet{UAI_CAN} infer confounders of univariate $X,Y$ subject to the additive noise assumption.} from $P_{\bX,Y}$ alone 
we are aware of are given by \citet{multivariateConfound,JanSch18}.
The first paper considers only one-dimensional confounders, which is complementary to our confounding scenario, while we will use the approach from the second paper 
because it perfectly matches our scenario 2 in Section~\ref{sec:overfitconfound} with fixed $M$. 
\citet{JanSch18} use 
 the slightly stronger assumption that $\ba$ and $\bc$ are drawn from $\cN(0,\sigma^2_a I)$ and $\cN(0,\sigma^2_c I)$, respectively. We briefly rephrase the method.
 
The idea is that the unregularized regression vector
$\tilde{\ba}$ in \eqref{eq:atilde} 
follows the distribution $\cN(0,\sigma_a^2 {\bf I} + \sigma_c^2 M^{-1}M^{-T})$. This results from 
\[
\tilde{\ba} = \ba + X^{-1} E = \ba + M^{-1} \bc,
\]
(see proof of Theorem~1 by \citet{JanSch18}). 
Then the quotient $\sigma_c^2/\sigma^2_a$   
can be inferred from the direction of $\hat{\ba}$ (intuitively: the more $\hat{\ba}$ concentrates in small eigenvalue eigenspaces of $\sx=M^T M$, the larger is this quotient).
Using some approximations that hold for large $d$, the confounding strength
\begin{equation}\label{eq:beta}
\beta:= \frac{\|\tilde{\ba}-\ba\|^2}{\|\tilde{\ba}-\ba\|^2 + \|\ba\|^2} \in [0,1]
\end{equation}
can be estimated from the input $(\widehat{\sx},\ba')$.
\citet{JanSch18} already observed the analogy between overfitting and confounding and also used the algorithm to estimate overfitting in scenario 1, which inspired this work.  
Using the approximation $\|\tilde{\ba}-\ba\|^2 + \|\ba\|^2\approx \|\tilde{\ba}\|^2$ \cite{multivariateConfound}, we have
$
\|\ba\|^2 \approx (1-\beta) \cdot \|\tilde{\ba}\|^2.
$
Hence, the length of the true causal regression vector $\ba$ can be estimated from the length $\tilde{\ba}$. 
This way, we can adjust $\lambda$ such that $\|\hat{\ba}_\lambda\|$ coincides with the estimated length. Since the estimation is based on a Gaussian (and not a Laplacian) prior for $\ba$,
it seems more appropriate to combine it with Ridge regression than with Lasso. However, since Lasso regression is known to have important advantages\footnote{\citet{Tibshirani2015} claim, for instance, ``If $\ell_2$ was the norm of the 20th century, then $\ell_1$ is the norm of the 21st century ... OK, maybe that statement is a bit dramatic, but at least so far, there’s been a frenzy of research involving the $\ell_1$ norm and its sparsity-inducing properties....''}  
(e.g. that sparse solutions yield more interpretable results), we also use Lasso. 
After all, the qualitative statement that strong confounding amounts to vectors $\hat{\ba}$ that tend to concentrate in low eigenvalue subspaces of $\sx$ still hold true
when $\bc$ is chosen from an isotropic prior. 

Confounding estimation via the algorithm of \citet{JanSch18} requires the problematic decision of whether the variables $X_j$ should be rescaled to variance $1$. 
If different $X_j$ refer to different units, there is no other straightforward choice of the scale. It is not recommended, however, to always normalize $X_j$. If $\sx$ is diagonal, for instance, the method would be entirely spoiled by normalization. 
The difficulty of deciding whether data should be renormalizing as an additional preprocessing step will be inherited by our algorithm. 

Our confounder correction algorithm reads:

%\begin{algorithm}[!htb]
%\label{alg:eststrength}

\vspace{0.2cm}
{\bf ConCorr}
\begin{algorithmic}[1]
  \STATE {\bfseries Input:} I.i.d. samples from
$P(\bX,Y)$.  

\STATE Rescale $X_j$ to variance $1$ if desired.

\STATE Compute the empirical covariance matrices
$\widehat{\sx}$ and $\widehat{\Sigma_{\bX Y}}$
   \STATE Compute the regression vector 
$
\hat{\ba}:=\widehat{\sx}^{-1} \widehat{\Sigma_{\bX Y}
}$

\STATE Compute an estimator $\hat{\beta}$ for the confounding strength $\beta$ via the algorithm in \cite{JanSch18} from  $\widehat{\sx}$ and 
$\hat{\ba}$ and estimate the squared length of $\ba$ via
\begin{equation}\label{eq:sqlengtha}
\|\ba\|^2 \approx (1-\hat{\beta}) \|\hat{\ba}\|^2
\end{equation}
\STATE
find $\lambda$ such that the squared length of  $\hat{\ba}^{\tt lasso/ridge}_\lambda$ coincides with the right hand side of \eqref{eq:sqlengtha}

\STATE Compute Ridge or Lasso regression model using this value of $\lambda$ 

   \STATE {\bfseries Output:} Regularized regression vector $\ba_\lambda$ 
\end{algorithmic}
%\end{algorithm}

\section{Experiments \label{sec:exp}}

\subsection{Simulated data}
We have generated data in a way that admits moving between scenarios 1 and 2 by simply changing some parameters.
For some fixed values of $d$ and $\ell$, we generate one mixing matrix $M$ in each run by drawing its entries from the standard normal distribution. In each run we generate $n$ instances of the $\ell$-dimensional standard normal random vector $\bZ$ and compute the $\bX$ values by $\bX = \bZ M$.  Then we choose $\sigma_c$, the parameter that
crucially controls confounding: for $\sigma_c=0$, we obtain scenario 1. For scenario 2, we choose $\sigma_c$ uniformly at random from $[0,1]$. Likewise, we draw $\sigma_a$, the parameter that controls the strength of the causal influence, uniformly at random from $[0,1]$. 
Then we draw the entries of $\bc$ and $\ba$ from $\cN(0,\sigma_c^2)$ and $\cN(0,\sigma^2_a)$, respectively, and compute the values of $Y$ via
$
Y = \bX \ba + \bZ \bc +E,
$
where $E$ is random noise drawn from $\cN(0,\sigma^2_E)$ (the parameter $\sigma_E$ 
% called F earlier and in the code!
has previously been chosen uniformly at random from $[0,5]$, which yields quite noisy data). While such a  noise term didn't exist in our description of scenario 2,
we add it here to study typical finite sample effects (without noise, $Y$  depends deterministically on $\bX$ for $\ell\leq d$).

To evaluate the  performance of causal regularization we define the relative squared error of any regression vector $\ba'$ by 
\[
\epsilon_{\ba'} := \frac{\|\ba' -\ba\|^2}{\|\ba' -\ba\|^2 + \|\ba\|^2}.    
\]
Note that $\|\ba' -\ba\|^2 + \|\ba\|^2\approx \|\ba'\|^2$ whenever the error $\|\ba' -\ba\|^2$ is close to orthogonal to $\ba$, which is a priori likely for vectors in high dimensions. Then $\epsilon_{\ba'} \approx \|\ba' -\ba\|^2/ \|\ba'\|^2$, which justifies the name `relative squared error'.    

For the special case where $\ba'$ is the unregularized regression vector $\hat{\ba}$ in \eqref{eq:lsq}, we define
\[
\epsilon_{\rm unreg} := \frac{\|\hat{\ba} -\ba\|^2}{\|\hat{\ba} -\ba\|^2 + \|\ba\|^2},
\]
by slightly overloading notation.  In the infinite sample limit $\epsilon_{\rm unreg}$ converges to the confounding strength  $\beta$, see \eqref{eq:beta}.

We begin with the unconfounded case $\sigma_c=0$ with $d=\ell=30$ and $n=50$. Figures~\ref{fig:myridge_cvridge}, left and right, show the relative squared errors obtained by our method ConCorr over the 
unregularized errors. The red and green lines show two different baselines: first, the unregularized error, and second, the error $1/2$ obtained by the trivial regression vector $0$. The goal is to stay below both baselines. Apart from those two trivial baselines, another natural baseline is regularized regression where $\lambda$ is chosen by cross-validation, because this would be the default approach for the unconfounded case. We have used leave-one-out CV from the {\tt Python} package {\tt scikit} for Ridge and Lasso, respectively.   
\begin{figure}
\centerline{
\includegraphics[width=0.2\textwidth]{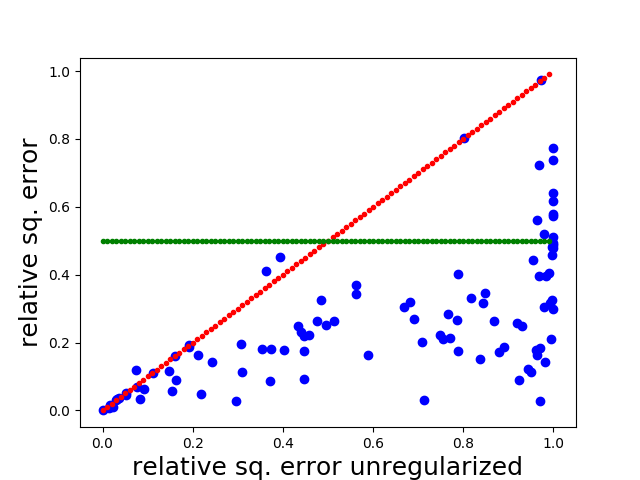}
\includegraphics[width=0.2\textwidth]{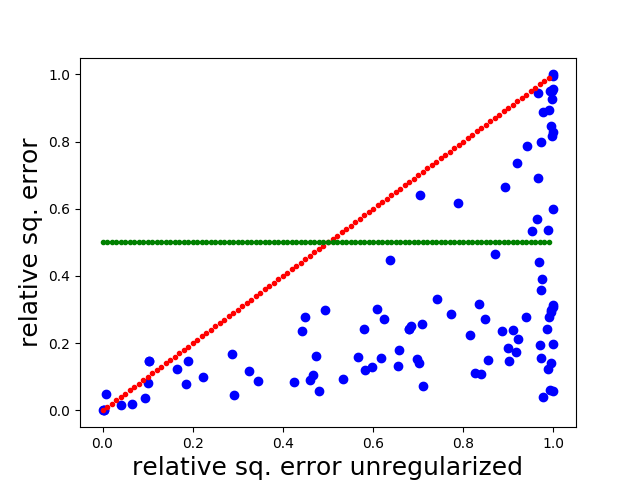}
}
\centerline{
\includegraphics[width=0.2\textwidth]{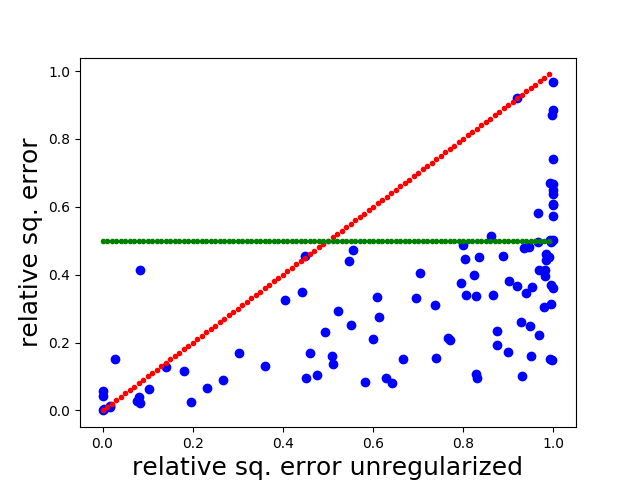}
\includegraphics[width=0.2\textwidth]{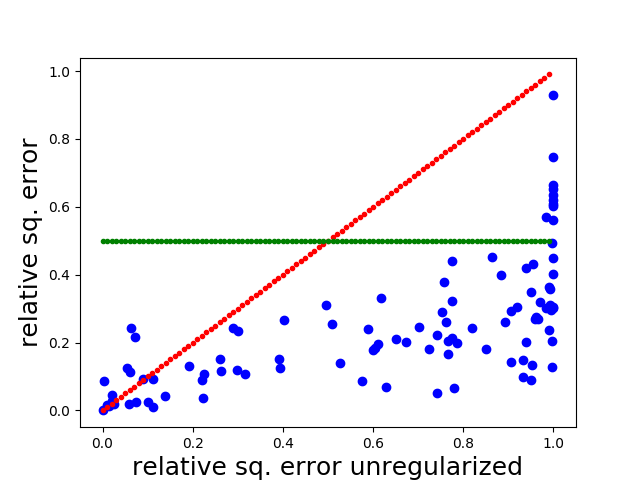}
}
\caption{\label{fig:myridge_cvridge} Results for Ridge (top) and Lasso (bottom) regression with ConCorr (left) versus cross-validated version (right) for the unconfounded case where artifacts are only due to overfitting. The results are roughly the same.}
\end{figure}
To quantitatively evaluate the performance, we have defined the {\it success rate} as the fraction of cases in which the relative squared error is at least by $0.05$ below both baselines\footnote{Note that this is a quite demanding criterion for success because there is no obvious way to decide which one of the two baseline methods performs better when $\ba$ is not known.}, the unregularized relative squared error and the value $1/2$. Likewise, we define the {\it failure rate} as the fraction of cases where the relative squared error is by $0.05$ larger than at least one of the baselines.  
We obtained the following results: 

\begin{tabular}{lcc}
{\bf method} & {\bf successes} & {\bf failures}\\
ConCorr Ridge/Lasso & $0.63/0.61$ & $0.11/0.16$\\
CV Ridge/Lasso & $0.65/0.72$ & $0.21/0.15$ \\
%ConCorr Lasso & $0.61$ & $0.16$ \\
%CV Lasso & $0.72$ & $0.15$
\end{tabular}

%In the case of Ridge regression, ConCorr works slightly better than CV, which suggests that for our generating model ConCorr detects 
%overfitting sufficiently well to select $\lambda$. Since inferring overfitting via ConCorr requires Gaussian noise we would not derive the recommendation
%to prefer ConCorr to CV for the unconfounded scenario. 
The results are roughly comparable, if we abstain from over-interpretations.  In the regime where the unregularized relative squared error is around $1/2$, all $4$ methods yield errors that are most of the time significantly lower. All $4$  methods have problems in the regime where the unregularized error is close to $1$ and sometimes
regularize to little for these cases.
 
To test the performances for scenario 2, we considered the  large sample case ($n=1000$, $d=30$) with confounding, where we obtained the results in Figure~\ref{fig:myridge_cvridge_con}. 
\begin{figure}
\centerline{
\includegraphics[width=0.2\textwidth]{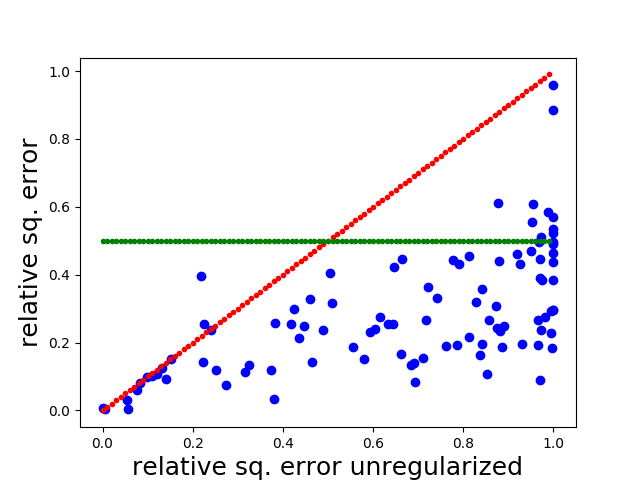}
\includegraphics[width=0.2\textwidth]{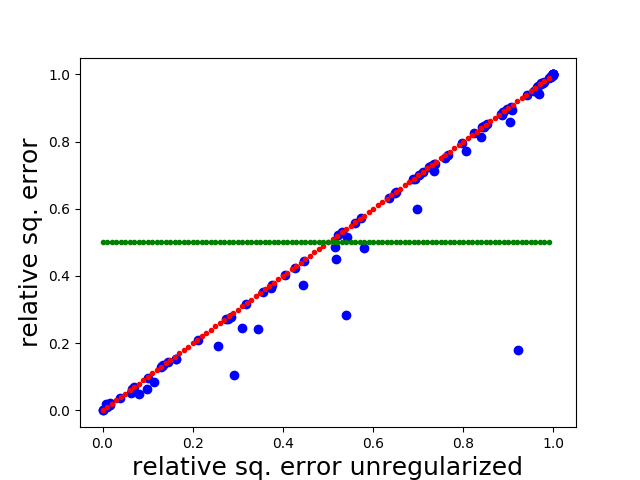}
}
\centerline{
\includegraphics[width=0.2\textwidth]{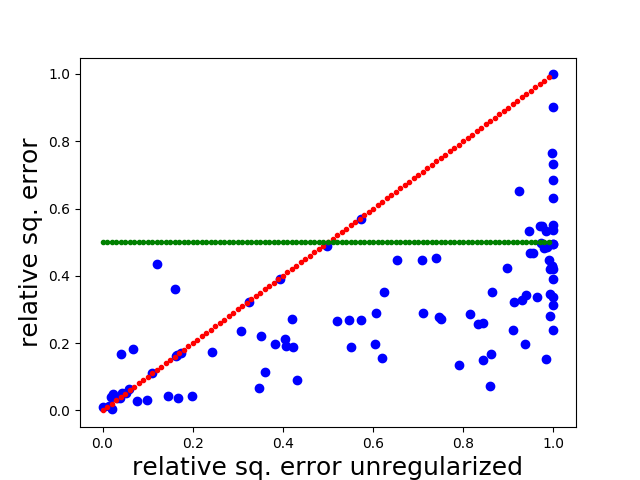}
\includegraphics[width=0.2\textwidth]{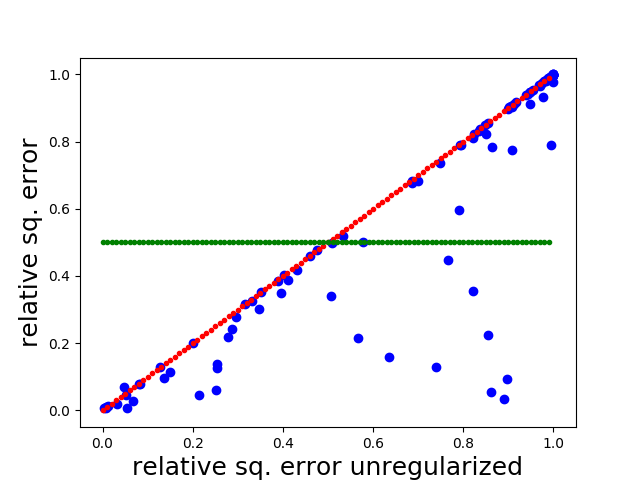}
}
\caption{\label{fig:myridge_cvridge_con} Results for Ridge (top) and Lasso (bottom) regression with ConCorr (left) versus cross-validated version (right) for the confounded case with large sample size where artifacts are almost only due to confounding. }
\end{figure}
Remarkably, ConCorr performed quite quite well also with Lasso regression although the Laplacian prior  of Lasso does not match our data generating process where the vector $\ba$ has been chosen from a Gaussian. 
One may argue that experiments for the confounded large sample regime are pointless since our theory states the equivalence of scenario 1 and 2. We show the experiments nevertheless for two reasons. First, it is not obvious which sample size approximates the population limit sufficiently well, and second,
 we have, by purpose, not chosen the parameters for generating $M$ according to the theoretical correspondence in order not to repeat equivalent experiments. The success and failure rates read:

\begin{tabular}{lcc}
{\bf method} & {\bf successes} & {\bf failures}\\
ConCorr Ridge/Lasso & $0.69/0.55$ & $0.08/0.13$\\
CV Ridge/Lasso & $0.08/0.15$ & $0.56/0.54$ \\
%ConCorr Lasso & $0.55$ & $0.13$ \\
%CV Lasso & $0.15$ & $0.54$
\end{tabular}

\begin{figure}
\centerline{
\includegraphics[width=0.2\textwidth]{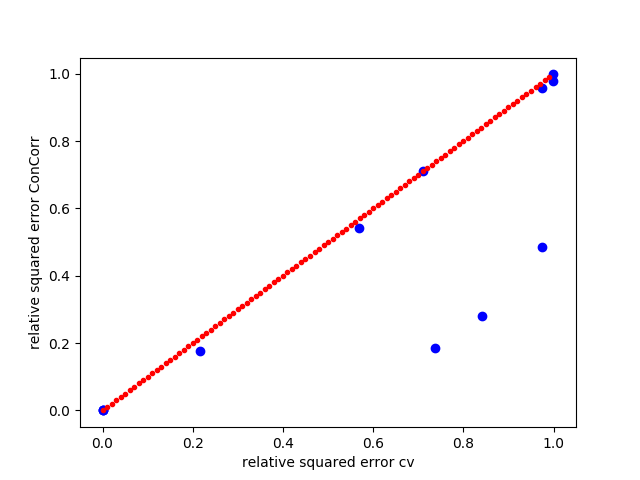}
\includegraphics[width=0.2\textwidth]{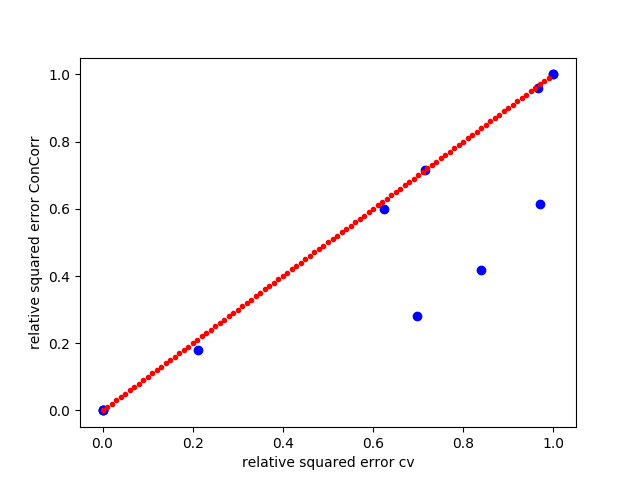}}
\caption{\label{fig:optical} Results for Ridge (left) and Lasso (right) regression for the data from the optical device in \citet{JanSch18}. The $y$-axis is the relative squared error achieved by ConCorr, while the $x$-axis is the cross-validated baseline.}
\end{figure}

Here, ConCorr clearly outperforms cross-validation (for both Ridge and Lasso), which shows that cross-validation
 regularizes too weakly for causal modelling, as expected.
One should add, however, that we increased the number of iterations in the $\lambda$-optimization to get closer to optimal leave-one-out performance
since the default parameters of {\tt scikit} already resulted in regularizing more strongly than that. 
%\footnote{`Early stopping' also regularizes, although 
%here the $\lambda$-optimization is stopped earlier instead of the regression itself as in \citep{Raskutti2011}.}
%This is consistent with the general insight that bounding running time also has a regularizing effect, although it is less easy to understand theoretically. 
Note that the goal of this paper is not to show that ConCorr outperforms other methods. Instead, we want to argue that for causal models 
it is often recommended to regularize more strongly than 
criteria of {\it statistical predictability} suggests. If `early stopping' in common CV algorithms also yields stronger regularization,\footnote{See also \citep{Raskutti2011} for regularization by early stopping in a different context.} 
%by early stopping, 
this can be equally helpful for causal inference, although
the way ConCorr choses $\lambda$ is less arbitrary than just bounding the number of iterations. 

We briefly mention results for the confounded case with small sample size, a regime for which we have no theoretical results. 
%The success and failure rates read:
%\begin{tabular}{lcc}
%{\bf method} & {\bf successes} & {\bf failures}\\
%ConCorr Ridge & $0.66$ & $0.14$\\
%CV Ridge & $0.20$ & $0.54$ \\
%ConCorr Lasso & $0.65$ & $0.13$ \\
%CV Lasso & $0.59$ & $0.23$
%\end{tabular}
Here, CV Lasso performs comparably to ConCorr, which is probably due to the strong finite sample effects.

We also checked how the performance depends on the dimension, but one should not overestimate the value of these experiments since
the estimation of confounding already depends heavily on  the distribution of eigenvalues of $\sx$.

\subsection{Real data}

To get confounded real data with $\bX$ and $Y$ being linked by a linear causal relation with {\it known} regression vector is not easy.
One approach is to restrict an unconfounded model
to a subset of variables: whenever $Y=\bX\ba +E$ is unconfounded, the statistical relation between $Y$ and a subset of $\bX$ can become confounded
by dropping parts of $\bX$  that influence  $Y$ (if the dropped components and the remaining ones have a common cause or the dropped ones influence the remaining ones).  
The true causal regression vector for the reduced system is given by simply reducing $\ba$ to the respective components (if the sample size is large enough 
to avoid overfitting, which we assume below in agreement with \cite{JanSch18}). 
However, to find multivariate data that is known to be unconfounded is difficult too.
 
 \paragraph{Optical device}  For this reason, \citep{JanSch18} have build an optical device where the screen of a Laptop shows an image with extremely low resolution (in their case $3\times 3$-pixel\footnote{In order to avoid overfitting issues \citet{JanSch18} decided to only generate low-dimensional data with $d$ around $10$.}) captured from a webcam. In front of the screen they mounted a photodiode measuring the light intensity $Y$, which is mainly influenced by the pixel vector $\bX$ of the image.  

As confounder $W$ they generated a random voltage controlling two LEDs, one in front of the webcam (and thus influencing $\bX$) and the second one in front of the
photodiode (thus influencing $Y$).  Since $W$ is also measured, the vector $\ba_{\bX,W}$ obtained by regressing $Y$ on $(\bX,W)$ is causal (no confounders by construction), if one accepts the linearity assumption. Dropping $W$ yielded significant confounding, with $\beta$ ranging from $0$ to $1$. We applied ConCorr to $\bX,Y$ and compared the output with the ground truth.
Figure~\ref{fig:optical}, left and right, show the results for Ridge and Lasso, respectively. 
The point $(0,0)$ happened to be met by three cases, where no improvement was possible. Fortunately, ConCorr did not make the result worse.
One can see that in $3$ out of the remaining nine cases (note that the point $(1,1)$ is also met by two cases),  ConCorr significantly improved the causal prediction. Fortunately, there is no case where ConCorr is worse than the baseline.

\paragraph{Taste of wine}
\citep{JanSch18}, moreover, used a dataset from the UCI machine learning repository \citet{Newman1998} of which they believe that it is almost unconfounded. $\bX$ contains $11$
ingredients of different sorts of red wine and $Y$ is the taste assigned by human subjects. Regressing $Y$ on $\bX$ yields a regression vector for which
the ingredient {\tt alcohol} dominates. Since alcohol strongly correlates with some of the other ingredients, dropping it amounts to significant confounding (assuming that the correlations between alcohol and the other ingredients is due to common causes and not due to the influence of alcohol on the others).

After normalizing the ingredients\footnote{Note that \citet{JanSch18} also used normalization to achieve reasonable estimates of confounding for this case.}, 
ConCorr with Ridge and Lasso yielded  a relative error of  $0.45$ and $0.35$, respectively, while \cite{JanSch18} computed the confounding strength $\beta\approx 0.8$, which means that ConCorr significantly corrects for confounding (we confirmed that CV also yielded errors close to $0.8$ which
suggests that finite sample effects did not matter for the error).

\section{Learning theory on `generalization' from observational to interventional distributions \label{sec:learningtheory}}

So far, we have motivated causal regularization mainly via transferring Bayesian arguments for regularization from scenario 1 to scenario 2. An alternative
perspective on regularization is provided by statistical learning theory \citep{Vapnik}. Generalization bounds guarantee that the {\it expected} error is unlikely to significantly exceed the {\it empirical} error for any regression function $f$ from a not too rich class $\cF$.
We will argue that our analogy between overfitting and confounding can be further extended to translate
generalization bounds in a way that they bound the error made by the {\it causal interpretation of regression models} when they are taken from a not too rich model class.
To make the analogy as natural as possible, we rephrase usual generalization bounds as:
\begin{center}
error of $f$ w.r.t. true (observational) distribution \\
$\leq$ error of $f$ w.r.t. empirical distribution + $C(\cF)$,
\end{center}
where $C(\cF)$ is some `capacity' term that accounts for the richness of the class $\cF$. 
%We have argued in Section~\ref{sec:overfitconfound} that inferring properties of the true distribution from the empirical one
%in scenario 1 is analog to inferring properties of the interventional distribution from the observational distribution in scenario 2.  
%Accordingly, we
Then we 
expect, subject to some conditions on the confounder, `causal generalization bounds' of the form\footnote{This kind of `causal learning theory' should not be confused with the one developed in \cite{Lopez2015} which considers algorithms
that infer cause vs effect from bivariate distributions after getting sufficiently many data sets with cause-effect pairs as training data. 
The cause-effect problem then reduces to a binary classification problem with bivariate empirical distribution as feature.
Our learning theory deals with a single data set.}
:
\newpage
\begin{center}
error of $f$ w.r.t. interventional distribution\\
$\leq$ error of $f$ w.r.t. observational distribution + $C(\cF)$.
\end{center}

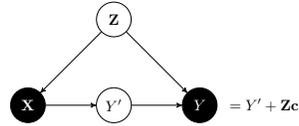
\begin{figure}
\begin{center}
\resizebox{4cm}{!}{
    \begin{tikzpicture}
    \node[obs] at (0,0) (X) {$\bX$} ;
   % \node[anchor=center] at (-1.2,0) {$\bZ M =$};
    \node[var] at (2,0) (Y') {$Y'$} edge[<-] (X);
    \node[obs] at (4,0) (Y) {$Y$} edge[<-] (Y');
    \node[anchor=center] at (5.5,0) {$ = Y' + \bZ\bc$};
    \node[var] at (2,2) (Z) {$\bZ$} edge[->] (X) edge[->] (Y); 
   % \node[anchor=center] at (-1.3,2) {$\cN(0,I)  \sim$};
  \end{tikzpicture}
}
\caption{\label{fig:confGen} Our confounding scenario: the high-dimensional common cause $\bZ$ influences $Y$ in a linear additive way, while the influence on $\bX$ is arbitrary.} 
\end{center}
\end{figure}

Figure~\ref{fig:confGen} shows our confounding model that significantly generalizes our previous models. $\bZ$ and $\bX$ are arbitrary random variables of dimensions $\ell$ and $d$, respectively. 
Apart from the graphical structure, we only add the
parametric assumption that the influence of $\bZ$ on $\bY$ is linear additive:
\begin{equation}\label{eq:shiftmodel}
Y =  Y' + \bZ\bc,
\end{equation}
where $\bc\in \R^{\ell}$.   
%While $Y'$ is unobservable, the observable implication of \eqref{eq:shiftmodel} reads:
%\begin{equation}\label{eq:shiftmodelCon}
%(Y - \bZ \bc)   \independent \bZ \,| \bX.
%\end{equation}
The change of $Y$ caused by setting $\bX$ to $\bx$ via interventions is given by Pearl's backdoor criterion \citep{Pearl:00} via
\[
p(y| do(\bx)) = \int p(y|\bx,\bz)p(\bz) d\bz.  
\]
For any function $f:\R^d \to \R$ we want to quantify how well it captures the behavior of $Y$ under interventions on $\bX$ and introduce the
{\it interventional loss}
\begin{equation}\label{eq:definter}
\Exp_{do(\bX)} [ (Y - f(\bX))^2 ] := \int (y - f(\bx))^2  p(y|do(\bx)) p(\bx) d\bx. 
\end{equation}
We want to compare it to the {\it observational loss}
\begin{equation}\label{eq:defobser}
\Exp [ (Y - f(\bX))^2 ] = \int (y - f(\bx))^2   p(y|\bx) p(\bx) d\bx. 
\end{equation}
In other words, we compare the expectations of the random variable $(Y-f(\bX))^2$ w.r.t.~ the distributions
$
p(y,\bx)$ and $p(y|do(\bx)) p(\bx)$.
The appendix shows that the difference between \eqref{eq:definter} and \eqref{eq:defobser} can be concisely phrases in terms of covariances:
\begin{Lemma}[interventional minus observational loss]
Let $g(\bx):= \Exp[Y'|\bx]$. Then 
\[
\Exp_{do(\bX)} [ (Y - f(\bX))^2 ]  - \Exp [ (Y - f(\bX))^2 ] =
(\Sigma_{(f- g)(\bX) \bZ}) \bc.
\] 
\end{Lemma}
For every single $f$, the vector  $\Sigma_{(f-g)(\bX)\bZ}$ is likely to be almost orthogonal 
to
$\bc$ if $\bc$ is randomly drawn from a rotation invariant distribution in $\R^\ell$. In order to derive statements of this kind that hold {\it uniformly} for all functions 
from a function class $\cF$ we introduce the following concept quantifying 
the richness of $\cF$:
\begin{Definition}[correlation dimension]
Let $\cF$ be some class of functions $f:\R^d \to \R$. 
Given the distribution $P_{\bX,\bZ}$, the correlation dimension $d_{\rm corr}$ of $\cF$
is the dimension of the span of
\[
\{ \Sigma_{f(\bX) \bZ} \,\,| f\in \cF \}. 
\]
\end{Definition}
To intuitively understand this concept it is instructive to consider the following immediate bounds:
\begin{Lemma}[bounds on correlation dimension]
The correlation dimension of $\cF$ is bounded from above by the dimension of the span of $\cF$.
 Moreover, if $\cF$ consists of linear functions, another upper bound is given by the rank of $\Sigma_{\bX \bZ}$. 
\end{Lemma}
%We will later need the following terminology.
%\begin{Definition}[random linear combination with fixed variance]
%Let $\bZ$ be an $\ell$-dimensional variable with covariance matrix $\bZ$. Then a `random linear combination' with variance $V$ is
%given s follows. Define $\bZ':= V^{1/2} (\Sigma_\bZ)^{-1/2} \bZ$, then draw a vector $\bc\in \R^\ell$ uniformly from the unit sphere   
%(in the sense of the Haar measure for the orthogonal group $O(\ell)$). Then define the `random linear combination' $W:= \bZ' \bc$, which
%has clearly variance $V$.
%\end{Definition}
%In other words, after `whitening' the confounder to have isotropic covariance matrix, we choose a random vector with unit length to define the linear combination.

In the appendix I show:
\begin{Theorem}[causal generalization bound] \label{thm:gen}
Given the causal structure in Figure~\ref{fig:confGen}, where $\bZ$ is $\ell$-dimensional with covariance matrix $\sz = {\bf I}$,
influencing $\bX$ in an arbitrary way. 
Let the influence of $\bZ$ on $Y$ be given by a `random linear combination' of $\bZ$ with variance $V$. Explicitly, 
\[
Y' \mapsto Y = Y' + \bZ\bc, 
\]
where $\bc\in \bc^\ell$ is randomly drawn from the sphere of radius $\sqrt{V}$ according to the Haar measure of $O(\ell )$.
Let $\cF$ have correlation dimension $d_{\rm corr}$ and satisfy the bound
$\|(f-g)(\bX)\|_\cH \leq b$ for all $f\in \cF$ (where $g(\bx):=\Exp[Y'|\bx]$). 
Then, for any $\beta>1$,
\begin{align*}
&\Exp_{do(\bX)} [ (Y - f(\bX)^2 ] \leq \Exp [ (Y - f(\bX))^2 ] \\
&+  b \cdot \sqrt{V \cdot \beta \cdot \frac{d_{\rm corr}+1}{\ell}},  
\end{align*}
holds uniformly for all $f\in \cF$  
with probability $e^{n (1-\beta +\ln \beta)/2}$. 
\end{Theorem}
Note that  $\sz = {\bf I}$ can always be achieved by the `whitening' transformation $\bZ \mapsto (\sz)^{-1/2} \bZ$. Normalization is convenient just because it enables a simple way to define a `random linear combination  of $\bZ$ with variance $V$', which would be cumbersome to define otherwise. 

Theorem~\ref{thm:gen} basically says that the interventional loss 
is with high probability close to the expected observational loss whenever the number of sources significantly exceeds the correlation dimension.
Note that the confounding effect can nevertheless be large.
Consider, for instance, the case where $\ell=d$ and $\bX$  and $\bZ$ are related by $\bX=\bZ$. Let, moreover, $Y'=\bX \ba$ for some $\ba\in \R^d$. 
Then the confounding can have significant impact on the correlations between $Y$ and $\bX$ due to 
$Y = \bX (\ba + \bc)$, whenever $\bc$ is large compared to $\ba$.
However, whenever $\cF$ has low correlation dimension, the selection of the function $f$ that optimally fits observational data is not significantly 
perturbed by the term $\bX\bc$. This is because $\bX\bc$ `looks like random noise' since $\cF$ contains no function that is able to account for
`such a complex correlation'. 
Since $\ell,d_{\rm corr}, b$ in Theorem~\ref{thm:gen} are unobserved, its value will mostly consist in qualitative insights rather than
providing quantitative bounds of practical use.

\section{What do we learn for the general case? \label{sec:outlook}}

%While the tight analogy between overfitting and confounding (observed by \citet{JanSch18} and further explored here) certainly relies on an oversimplified
%model of confounding, we believe that our results contain some messages for real world scenarios. 
%In statistical inference, Vapnik Chervonencis learning theory has shown that generalization {\it across different subsamples} from the same distributions
%requires regularization in the sense of choosing models from not too rich model classes. Here we have described a scenario where regularization also helps for generalizing across different background conditions, which is, in our case, equivalent to `generalizing' from the statistical model to the causal model.
% cite invariant prediction, Storkey...
Despite all concerns against our oversimplified assumptions, I want to stimulate a general discussion about recommending
stronger regularization than criteria of {\it statistical predictability} suggest -- whenever one is actually interested in {\it causal} models, which
are more and more believed to be required for generalization across different domains \citep{Peters15,Zhang2017CausalDF,Heinze-Deml2017,causal_image_classification}.
It is, however, by no means intended to suggest that this simple recommendation would {\it solve} any of the hard problems in causal inference.
%strongly regularized statistical inference procedures have better chances to provide {\it causal} insights. 
%One may derive the general recommendation to regularize more than the respective sample size required given that one is interested in causal results. 
%Since modern deep learning algorithms often achieve impressive results in predictive tasks (using features that are hardly accessible to human interpretation),
%the question of the causal content of the predictive models gets even more series.  

%\bibliographystyle{icml2019}
%\bibliography{../../literature/literatur}

\section{Appendix}

\subsection{Inferring $\ba$ from covariance matrices alone}

The following result shows that standard Ridge and Lasso regression can be rephrased in a way, that they receive only 
empirical covariance matrices $\widehat{\sx},\widehat{\sxy}$ as input. Likewise our population Ridge and Lasso only require population covariance matrices 
$\sx,\sxy$ 
as input:
\begin{Lemma}[inferring the vector $\ba$ from covariances]\label{lem:acov}
The posterior probabilities 
%\eqref{eq:postridge} 
(13)
and 
%\eqref{eq:postlasso} 
(14)
can be equivalently written as
\begin{align*}
\log &\, p_{\rm ridge}(\ba | \widehat{\sx}, \widehat{\sxy}) \\
&\stackrel{+}{=} -\frac{1}{2 \tau^2} \|\ba\|^2 -  (\ba - \hat{\ba})^T \widehat{\sx}^{-1}  (\ba -\hat{\ba}),\\
\log &\, p_{\rm lasso}(\ba | \widehat{\sx}, \widehat{\sxy}) \\
&\stackrel{+}{=} -\frac{1}{2 \tau^2} \|\ba\|_1 - (\ba - \hat{\ba})^T \widehat{\sx}^{-1}  (\ba -\hat{\ba}),
\end{align*}
and  $\hat{\ba}$ in 
%\eqref{eq:lsq} 
(2)
can be written as
\[
\hat{\ba} = \widehat{\sx}^{-1} \widehat{\sxy}.
\]
Likewise, the population versions (16) and (17) are equal to
\begin{align*}
\log &\, p_{\rm ridge}(\ba | \sx, \sxy) \\
&\stackrel{+}{=} -\frac{1}{2 \tau^2} \|\ba\|^2 -  (\ba - \tilde{\ba})^T \sx^{-1}  (\ba -\tilde{\ba}),\\
\log &\,  p_{\rm lasso}(\ba | \sx, \sxy) \\
&\stackrel{+}{=} -\frac{1}{2 \tau^2} \|\ba\|_1 - (\ba - \tilde{\ba})^T \sx^{-1}  (\ba -\tilde{\ba}),
\end{align*}
and  $\tilde{\ba}$ in %\eqref{eq:atilde}
(7)
 can be written as
\[
\tilde{\ba} = \sx^{-1} \sxy.
\]
\end{Lemma}
\begin{proof}
To rewrite $p_{\rm ridge} (\ba |\hat{\bX}, \hat{Y})$ and $p_{\rm lasso}(\ba |\hat{\bX}, \hat{Y})$ we note that for any $\ba'$,
\[
\|\hat{Y} - \hat{\bX}\ba'\|^2 = (\ba' - \hat{\ba})^T \widehat{\sx}^{-1}  (\ba' -\hat{\ba}) + \|\hat{Y}^\perp\|^2,
\]
where $\hat{Y}^\perp$ denotes the component of $\hat{Y}$ orthogonal to the image of $\hat{\bX}$, with $\hat{\ba}$ from 
%\eqref{eq:lsq}
(2).
Since the second term does not depend on $\ba'$, it is absorbed by the  normalization. 
The statement for the population versions follows similarly. 
\end{proof}

Using Lemma~\ref{lem:acov}, we can also directly justify the population versions of Ridge and Lasso without Theorem~2
%\ref{thm:justPopulation} 
by observing that they maximize posterior probabilities
of $\ba$ in scenario 2, provided that one is willing to accept the strong assumption from
Theorem~1.
%\ref{thm:equivalence}.  

\subsection{Proof of Theorem~3}
We first need 
the following result which is basically Lemma~2.2 in \citep{Dasgupta2003} together with the remarks preceding 2.2:
\begin{Lemma}[Johnson-Linderstrauss type result]\label{lem:JL}
Let $P$ be the orthogonal projection onto an $n$-dimensional subspace of $\R^m$
and $v\in \R^m$ be randomly drawn from the uniform distribution on the unit sphere. 
Then $\|Pv\|^2\geq  \beta n/m$  with probability at most
$e^{n (1-\beta +\ln \beta)/2}$. 
\end{Lemma}
We are now able to prove Theorem~3.
Let $\bc^{\cF} $ be the orthogonal projection of $\bc$ onto the span of $\{ \Sigma_{(g-f)(\bX)\bZ} \,| f\in \cF \}$ (whose dimension is at most $d_{\rm corr}+1$).  
Note that the vector $\Sigma_{(g-f)(\bX) \bZ} \in \R^\ell$ has the components  $\langle (g-f)(\bX), Z_j\rangle$ if $Z_j$ denotes the components of $\bZ$, 
which are orthonormal in $\cH$. Hence
\[
\|\Sigma_{(g-f)(\bX) \bZ}\| \leq b. 
\]
Thus
the absolute value of the difference of the losses is equal to
\begin{align*}
&  |\Sigma_{(g-f)(\bX)\bZ}  \bc^\cF | \leq b \sqrt{V}  \|\bc^\cF\|. 
\end{align*}
Then the proof follows from Lemma~\ref{lem:JL}.

\subsection{Poof of equation~(8)}
Due to $\sx  =\bX^T \bX$ we have
\[
\sx X^{-1} = X^\dagger X X^{-1} = X^\dagger,
\]
since $XX^{-1}$ is the orthogonal projection onto the image of $X$, which is orthogonal to the kernel of $X^T$. Then invertibility of $\sx$ implies
\[
X^{-1} E = \sx^{-1} X^T E = \sx^{-1} \sxe. 
\]

\subsection{Proof of Lemma~1}
Using definitions 
%\eqref{eq:defobser} 
(23)
and 
%\eqref{eq:definter} 
(22)
the difference between the two losses can be written as:
\begin{align*}
&\int (y-f(\bx))^2 [p(y|\bx) - p(y|do(\bx))] p(\bx) d\bx\\ 
&=\int (y-f(\bx))^2 p(y|\bx,\bz) \{p(\bx,\bz) - p(\bx)p(\bz)\} d\bz d\bx\\
&=\Exp[ (Y-f(\bX))^2 | \bx,\bz] \{ p(\bx,\bz) - p(\bx)p(\bz)\} d\bz d\bx.
\end{align*}
We rewrite the conditional expectation as
\begin{align}
&\Exp[ (Y - f(\bX))^2|\bx,\bz]  \label{eq:condExpexp}\\
&= \Exp[ (Y' + \bz \bc -f(\bx))^2   |\bx,\bz] \nonumber \\
&= \Exp[ Y'^2 |\bx,\bz] + (\bz\bc)^2 + f(\bx)^2  \nonumber\\
& + \Exp [Y' |\bx,\bz] \bz\bc  - \Exp[Y'|\bx,\bz] f(\bx) - f(\bx) \bz \bc. \nonumber \\
&=  \Exp[ Y'^2 |\bx] + (\bz\bc)^2 + f(\bx)^2  \nonumber\\
& + g(\bx) \bz\bc  - g(\bx) f(\bx) - f(\bx) \bz \bc,\nonumber 
\end{align}
where we have used 
%\eqref{eq:shiftmodelCon}
(20).
Since this conditional expectation is integrated over $p(\bx,\bz)-p(\bx)p(\bz)$, only terms matter that contain both $\bx$ and $\bz$. We therefore obtain
\begin{align*}
&\Exp [ (Y - f(\bX))^2 ] - \Exp_{do(\bX)} [ (Y - f(\bX))^2 ] \\
&=\int (g(\bx) -f(\bx)) \bz\bc   \{ p(\bx,\bz) - p(\bx)p(\bz)\} d\bz d\bx \\
&= (\Sigma_{(g-f)(\bX), \bZ}) \bc.
\end{align*}

\end{document}